%% file: main.tex
\documentclass[10pt,twocolumn,letterpaper]{article}

\usepackage{cvpr}              

\input{preamble}

\usepackage{graphicx}
\usepackage{amsmath}
\usepackage{amssymb}
\usepackage{booktabs}
\usepackage{multirow}

\usepackage{amsthm}
\newtheorem{theorem}{Theorem}[section]
\newtheorem{lemma}[theorem]{Lemma}

\definecolor{cvprblue}{rgb}{0.21,0.49,0.74}
\usepackage[pagebackref,breaklinks,colorlinks,citecolor=cvprblue]{hyperref}

\def\paperID{100} 
\def\confName{3DV\xspace}
\def\confYear{2024\xspace}

\title{Joint Spatial-Temporal Calibration for Camera and Global Pose Sensor}

\makeatletter
\def\thanks#1{\protected@xdef\@thanks{\@thanks
        \protect\footnotetext{#1}}}
\makeatother

\author{Junlin Song\thanks{This research was supported by the European Union’s Horizon 2020 project SESAME (grant agreement No 101017258). The authors are with the Space Robotics (SpaceR) Research Group, Int. Centre for Security, Reliability and Trust (SnT), University of Luxembourg, Luxembourg.} \and Antoine Richard \and Miguel Olivares-Mendez}

\begin{document}
\maketitle
\input{sec/0_abstract}    
\input{sec/1_intro}
\input{sec/2_method}
\input{sec/3_final}
{
    \small
    \bibliographystyle{ieeenat_fullname}
    \bibliography{main}
}
\input{sec/X_suppl}

\end{document}

%% file: preamble.tex
\usepackage[dvipsnames]{xcolor}


%% file: sec/0_abstract.tex
\begin{abstract}
In robotics, motion capture systems have been widely used to measure the accuracy of localization algorithms. 
Moreover, this infrastructure can also be used for other computer vision tasks, such as the evaluation of Visual (-Inertial) SLAM dynamic initialization, multi-object tracking, or automatic annotation.
Yet, to work optimally, these functionalities require having accurate and reliable spatial-temporal calibration parameters between the \textbf{camera} and the \textbf{global pose sensor}. 
In this study, we provide two novel solutions to estimate these calibration parameters. 
Firstly, we design an offline target-based method with high accuracy and consistency. Spatial-temporal parameters, camera intrinsic, and trajectory are optimized simultaneously.
Then, we propose an online target-less method, eliminating the need for a calibration target and enabling the estimation of time-varying spatial-temporal parameters.
Additionally, we perform detailed observability analysis for the target-less method. 
Our theoretical findings regarding observability are validated by simulation experiments and provide explainable guidelines for calibration.
Finally, the accuracy and consistency of two proposed methods are evaluated with hand-held real-world datasets where traditional hand-eye calibration method do not work.
\end{abstract}

%% file: sec/1_intro.tex
\section{Introduction}
\label{sec:intro}
Nowadays, motion capture systems are widely used to perform 6DoF pose tracking thanks to their high accuracy (sub-millimeter).
In odometry and SLAM research, most datasets leverage these to provide the ground truth pose \cite{burri2016euroc, schubert2018tum, delmerico2019we}. The collection platform from \cite{ schubert2018tum} shown in \cref{Frames in TUM-VI} displays some passive markers typically associated with motion capture systems. Aside from its application to localization methods, the potential of motion capture systems in the field of computer vision has not been fully exploited. The key is the spatial-temporal calibration parameters of the camera and the global pose sensor  (see \cref{spatial-temporal}). 

\begin{figure}[htbp]
  \centering
    \begin{subfigure}{0.23\textwidth}
        \centering
        \includegraphics[width=\textwidth, height=0.8\textwidth]{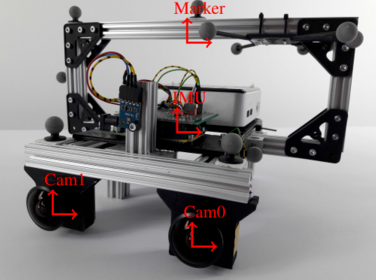}
        \caption{}
        \label{Frames in TUM-VI}
    \end{subfigure}
    \hfill
    \begin{subfigure}{0.23\textwidth}
        \centering
        \includegraphics[width=\textwidth, height=0.8\textwidth]{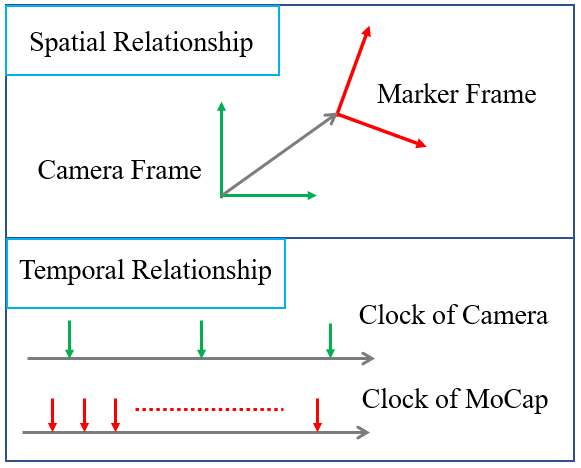}
        \caption{}
        \label{spatial-temporal}
    \end{subfigure}
  \caption{(a) Photo of the sensor setup, taken from \cite{schubert2018tum}. (b) The spatial-temporal relationship between the camera measurements and the global pose measurements.}
\end{figure}

For instance, in \cref{Frames for target-less method}, we assume a target tracking or automatic labeling task, performed with the motion capture system.
The camera $\left\{ C \right\}$ is rigidly linked with the marker frame $\left\{ M \right\}$ tracked by the motion capture system.
The target is regarded as a point $f$.
The motion capture system provides ${}^G{p_f}$ and $\left\{ {\begin{array}{*{20}{c}}
{{}_M^Gq}&{{}^G{p_M}}
\end{array}} \right\}$.
Given the spatial-temporal calibration parameters linked $\left\{ M \right\}$ and $\left\{ C \right\}$, the image coordinates of $f$ can be obtained automatically via rigid body link ($f \to G \to M \to C$).

The above example illustrates the benefits of having a spatial-temporal calibration between a camera and a global pose sensor. 
In the literature, the methods to solve the spatial-temporal calibration are divided into two categories: target-based methods and target-less methods.
The target-based methods are more accurate than the target-less methods, benefiting from the prior knowledge of the calibration target.
Target-based methods are widely used in multi-sensor calibration tasks \cite{furgale2013unified, rehder2016extending, rehder2016general}. 
Target-based spatial-temporal hand-eye calibration was first presented in \cite{furrer2018evaluation}.
The spatial-temporal parameters are calibrated by aligning the motion capture trajectory with the camera trajectory, which is obtained by the Perspective-n-Point (PnP) algorithm, with the calibration target.
The camera's intrinsic parameters are assumed to be fixed.
Therefore, the accuracy of \cite{furrer2018evaluation} is limited by the PnP algorithm, employed on every single image.
After the PnP process, all raw pixels measurements are discarded.
The isolation processing of the motion capture sequence and camera sequence cannot uncover the inherent correlation between raw pixel measurements and motion capture measurements. 
Unlike our target-based calibration algorithm which fully utilizes all the raw sensor data to optimize the spatial-temporal parameters, camera intrinsic and trajectory simultaneously.

However, these methods are only suitable for offline non-real-time calibration and require significant amounts of manual effort. 
Markers attached to the camera may be removed during experiments, therefore changing the spatial calibration parameter.
Moreover, the temporal calibration parameter would also change due to different clocks, transmission delays, data jam, jitter, and skew \cite{qiu2020real}. Therefore, online target-less calibration method is also worth exploiting, saving human effort and improving the ease of application.

In recent years, online target-less calibration has attracted significant attention in visual-inertial navigation systems (VINS) \cite{li2014online, qin2018online, yang2020online}.
Among them, the EKF-based methods are the most popular thanks to their computational efficiency. 
\cite{li2014online} pointed out that given sufficient motion excitation, the spatial-temporal calibration parameters of VINS are observable.
However, under specific motion profiles, some degrees of freedom of the calibration parameters would be unobservable \cite{yang2019degenerate}.
Identifying potential motion degradation, and avoiding such motion, is crucial to reliably apply these types of algorithms.

The contributions of this work are summarized as:

\begin{itemize}
    \item To our knowledge, this is the first work to simultaneously calibrate spatial-temporal parameters of the camera and the global pose sensor, with raw monocular camera pixel measurements and global pose measurements. 
    
    \item We propose two novel approaches to estimate the spatial-temporal parameters. Both target-based and target-less methods are considered.
    
    \item We provide detailed observability analysis for the proposed target-less calibration method and identify the degenerated motions that may occur in practice, causing partial calibration parameters unobservable.
    
    \item We verify the degenerate motions in simulation and evaluate the accuracy and consistency of two proposed algorithms with hand-held real-world datasets.

    \item We demonstrate the applicability of online calibration time-varying spatial-temporal parameters for the target-less method.
     
\end{itemize}

%% file: sec/2_method.tex
\section{Notation} \label{notation}

\begin{figure}[htbp]
  \centering
    \begin{subfigure}{0.23\textwidth}
        \centering
        \includegraphics[width=\textwidth, height=0.8\textwidth]{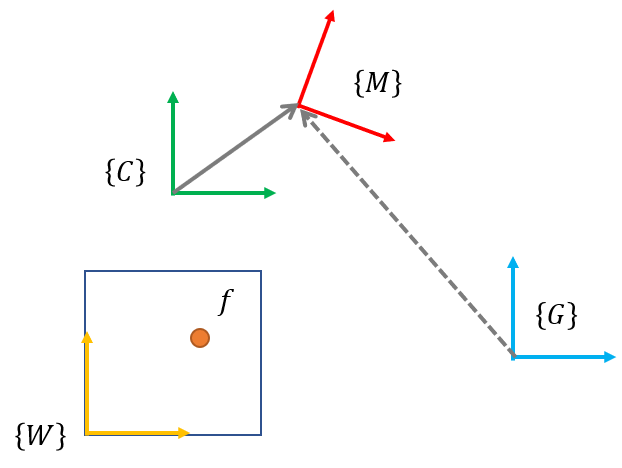}
        \caption{}
        \label{Frames for target-based method}
    \end{subfigure}
    \hfill
    \begin{subfigure}{0.23\textwidth}
        \centering
          \includegraphics[width=\textwidth, height=0.8\textwidth]{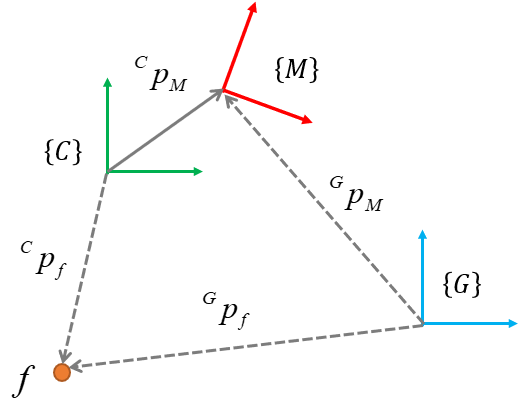}
          \caption{}
          \label{Frames for target-less method}
    \end{subfigure}
  \caption{(a) Coordinate frames for the target-based method. (b) Coordinate frames for the target-less method.}
  \label{Frames}
\end{figure}

As shown in \cref{Frames}, $\{ G\} $ represents the global reference frame of the motion capture system. $\{ M\} $ and $\{ C\} $ represent the marker frame and the camera frame respectively. In this paper, \textbf{“marker”} is an equivalent term of \textbf{“global pose sensor”}, as the 6DoF movement of frame $\{ M\} $ could be tracked by the motion capture system. The 6DoF rigid body transformation between $\{ M\} $ and $\{ C\}$, ${{}_M^CT}$, is the spatial calibration parameter. In our formulation, the camera time clock is treated as the time reference in the estimators. The time offset between the marker clock and the camera clock is the temporal calibration parameter ${t_d}$. If the timestamp at the camera clock is ${t_C}$, then the corresponding timestamp at the marker clock is:
\begin{equation}
    {t_M} = {t_C} + {t_d}
    \label{eq:clock}
\end{equation}

We use ${}^G\left(  \bullet  \right)$  to represent a physical quantity in the frame $\{ G\} $. The position of a point $M$ in the frame $\{ G\}$ is expressed as ${}^G{p_M}$. The velocity of a point $M$ in the frame $\{ G\}$ is expressed as ${}^G{v_M}$. The local angular velocity of $\{ M\} $ is denoted as $\omega $. A Unit quaternion is employed to represent the rotation of a rigid body \cite{trawny2005indirect}. ${}_G^Mq$ represents the orientation of the frame $\{ M\}$ with respect to the frame $\{ G\} $, and its corresponding rotation matrix is ${}_G^MR$. ${\left[  \bullet  \right]_ \times }$ is denoted as the skew symmetric matrix corresponding to a three-dimensional vector. The transpose of a matrix is ${\left[  \bullet  \right]^T}$. 

\section{Target-based Calibration}

A target-based calibration method which adopts offline full-batch nonlinear least squares optimization is designed to provide high accurate and consistent solutions for calibration parameters.

We use a grid of AprilTag \cite{olson2011apriltag} as the calibration target, as shown in \cref{fig:1 iter}. The coordinate frames involved in target-based method are depicted in \cref{Frames for target-based method}. Compared with \cref{Frames for target-less method}, additional frame $\{ W\} $ is built and fixed on the calibration target. 

Suppose that the timestamp of the $i$th image is ${t_i}$. The image coordinate of the $j$th AprilTag corner ${f_j}$ detected in the $i$th image is ${u_{ij}}$. Its associated 3D coordinates ${}^W{p_{{f_j}}}$ in $\{ W\} $ is known. The optimization variables are defined as:
\begin{equation}
    \chi  = \left\{ {\begin{array}{*{20}{c}}
    {{}_{{C_1}}^WT}& \cdots &{{}_{{C_N}}^WT}&{{}_W^GT}&{{}_M^CT}&{{t_d}}&{\varsigma}
    \end{array}} \right\}
    \label{eq:opt_var}
\end{equation}

Where $N$ is the image numbers. $\chi$ includes the all camera poses ${}_{{C_i}}^WT,{\rm{ }}i = 1 \cdots N$, the rigid body transformation between $\{ W\} $ and $\{ G\} $, the spatial-temporal calibration parameters $\left\{ {\begin{array}{*{20}{c}}
{{}_M^CT}&{{t_d}} \end{array}} \right\}$, and the vector of camera intrinsic parameters $\varsigma$. 
By integrating all raw image pixel measurements and global pose measurements, we formulate the least squares optimization as:
\begin{equation}
    \begin{array}{l}
    \chi  = \arg \min \left\{ {\sum\limits_{i = 1}^N {\sum\limits_{j = 1}^K {\rho \left( {{r_{ij}}} \right)}  + \sum\limits_{i = 1}^N {\rho \left( {{r_{gi}}} \right)} } } \right\}\\
    {r_{ij}} = \pi \left( {{}_W^{{C_i}}T{}^W{p_{{f_j}}}}, \varsigma \right) - {u_{ij}}\\
    {r_{gi}} = {Log} \left( {{}_G^MT\left( {{t_i} + {t_d}} \right){}_W^GT{}_{{C_i}}^WT{}_M^CT} \right)
    \end{array}
    \label{eq:min}
\end{equation}

Where $K$ is the corner numbers for each image. $\rho \left(  \bullet  \right)$ is a robust kernel function \cite{chebrolu2021adaptive}. $\pi \left(  \bullet, \bullet  \right)$ is a fixed camera projection function \cite{usenko2018double, heng2013camodocal}. $Log \left(  \bullet  \right)$ maps the element on a Lie group to the tangent space vector \cite{sola2018micro}.

${}_G^MT\left( {{t_i} + {t_d}} \right)$ is the interpolated global pose measurement. To calculate ${}_G^MT\left( {{t_i} + {t_d}} \right)$, we find two closet timestamps over all global pose measurements, $t_a$ and $t_b$, which subject to ${t_a} \le {t_i} + {t_d} < {t_b}$. Two corresponding pose measurements are ${}_G^{{M_a}}T$ and ${}_G^{{M_b}}T$ respectively. Using linear interpolation with two bounding poses, the synthetic measurement at ${t_i} + {t_d}$ is expressed as:
\begin{equation}
    \begin{array}{l}
    {}_G^MT\left( {{t_i} + {t_d}} \right) = Exp\left( {\lambda Log\left( {{}_G^{{M_b}}T{}_G^{{M_a}}{T^{ - 1}}} \right)} \right){}_G^{{M_a}}T\\
    \lambda  = {{\left( {{t_i} + {t_d} - {t_a}} \right)} \mathord{\left/
     {\vphantom {{\left( {{t_i} + {t_d} - {t_a}} \right)} {\left( {{t_b} - {t_a}} \right)}}} \right.
     \kern-\nulldelimiterspace} {\left( {{t_b} - {t_a}} \right)}}
    \end{array}
\end{equation}

$Exp\left(  \bullet  \right)$ is the inverse operation of $Log \left(  \bullet  \right)$ \cite{sola2018micro}.

Jacobians of residuals in \cref{eq:min} with respect to the optimization variables $\chi$ are calculated according to the chain rule and provided in \cref{sec:Jacobians} of supplementary material.
The Levenberg-Marquardt algorithm is adpot to minimize \cref{eq:min} and update the optimal estimation iteratively.

 Differentiate from \cite{furrer2018evaluation}, the proposed target-based method is able to optimize and refine the spatial-temporal calibration parameters, the transformation between $\{ W\} $ and $\{ G\} $, the camera intrinsic $\varsigma$ and trajectory ${}_{{C_i}}^WT,{\rm{ }}i = 1 \cdots N$ simultaneously, without information loss.

\section{Target-less Calibration}

To alleviate the need for calibration target and enable time-varying parameters calibration during the operation, we provide an alternative online EKF-based target-less calibration method. Coordinate frames are shown in \cref{Frames for target-less method}.

\subsection{State Vector}

The EKF state vector inspired by MSCKF \cite{mourikis2007multi} includes the marker state, the spatial-temporal calibration parameters, the camera intrinsic parameters, augmented $N$ marker states and up to $L$ augmented features:
\begin{equation}
    \begin{array}{l}
    x = {\left[ {\begin{array}{*{20}{c}}
    {x_M^T}&{x_{calib}^T}&{x_c^T}&{x_f^T}
    \end{array}} \right]^T}\\
    {x_M} = {\left[ {\begin{array}{*{20}{c}}
    {{}_G^M{q^T}}&{{}^Gp_M^T}&{{\omega ^T}}&{{}^Gv_M^T}
    \end{array}} \right]^T}\\
    {x_{calib}} = {\left[ {\begin{array}{*{20}{c}}
    {{}_M^C{q^T}}&{{}^Cp_M^T}&{{t_d}}&{\varsigma}
    \end{array}} \right]^T}\\
    {x_c}{\rm{ = }}{\left[ {\begin{array}{*{20}{c}}
    {x_{{c_1}}^T}& \cdots &{x_{{c_N}}^T}
    \end{array}} \right]^T}{\rm{  }}
    {\quad}{\quad}{x_{{c_i}}} = {\left[ {\begin{array}{*{20}{c}}
    {{}_G^{{M_i}}{q^T}}&{{}^Gp_{{M_i}}^T}
    \end{array}} \right]^T}\\
    {x_f}{\rm{ = }}{\left[ {\begin{array}{*{20}{c}}
    {{}^Gp_{{f_1}}^T}& \cdots &{{}^Gp_{{f_L}}^T}
    \end{array}} \right]^T}{\rm{ }}
    \end{array}
\end{equation}

Where ${x_M}$ is the current marker state at the camera clock. Calibration parameter ${x_{calib}}$ includes the 6DoF transformation $\left\{ {\begin{array}{*{20}{c}}{{}_M^Cq}&{{}^C{p_M}}\end{array}} \right\}$, the time offset ${t_d}$ and the camera intrinsic parameters $\varsigma$. ${x_c}$ is the augmented marker states, which is obtained by cloning the first two physical quantities of ${x_M}$ at different image times. $N$ is the sliding window size, a fixed parameter. The pose clones in the sliding window are utilized to triangulate environmental feature points. ${}^G{p_{{f_j}}}$ is an augmented feature, or termed as a SLAM feature \cite{li2013optimization, li2014visual, geneva2020openvins}.

Angular and linear velocity ($\omega$ and ${}^G{v_M}$) are included to predict the motion because the measurements provided by motion capture system may be intermittent. Moreover, they are needed to estimate time offset (see \cref{eq:H_aug}).

\subsection{Constant Velocity Propagation}

Referring to previous study on trajectory estimation \cite{dong2018sparse, schubert2018direct}, a constant-velocity motion prior is applied. ${x_M}$ is propagated forward based on the constant velocity motion model. The kinematic model can be described as:
\begin{equation}
    \begin{array}{l}
    {}_G^M\dot q = \frac{1}{2}\Omega \left( \omega  \right){}_G^Mq,{\quad}{}^G{{\dot p}_M} = {}^G{v_M}\\
    \dot \omega  = {n_\omega },{\quad}{}^G{{\dot v}_M} = {n_v}
    \end{array}
    \label{eq:kinematic model}
\end{equation}

$\Omega \left( \omega  \right) = \left[ {\begin{array}{*{20}{c}}
{ - {{\left[ \omega  \right]}_ \times }}&\omega \\
{ - {\omega ^T}}&0
\end{array}} \right]$. ${n_{\left[ \bullet \right]}}$  represents the zero mean Gaussian noise of $\left[  \bullet \right]$, which is a hyperparameter. These hyperparameters can be determinated in advance using existing approaches \cite{dong2018sparse, barfoot2014batch}.

By linearizing \cref{eq:kinematic model} at the current state estimation, the state transition matrix from time $t_0$ to time $t_{k}$ can be analytically calculated as follows:
\begin{equation}
    \begin{array}{l}
    {\Phi _M}\left( {{t_{k}},{t_0}} \right) = \left[ {\begin{array}{*{20}{c}}
    A&{{0_3}}&B&{{0_3}}\\
    {{0_3}}&{{I_3}}&{{0_3}}&{{I_3}\Delta t}\\
    {{0_3}}&{{0_3}}&{{I_3}}&{{0_3}}\\
    {{0_3}}&{{0_3}}&{{0_3}}&{{I_3}}
    \end{array}} \right]\\
    A = {}_G^{{M_{k}}} R{}_G^{{M_0}}{{ R}^T}\\
    B = {}_G^{{M_{k}}} R{}_G^{{M_0}}{{ R}^T}{J_r}\left( { - \omega \Delta t} \right)\Delta t
    \end{array}
    \label{eq:transition}
\end{equation}

Where ${J_r}\left( \bullet \right)$ is the right Jacobian of SO(3) \cite{barfoot2017state}. 

\subsection{Visual Measurement Update}

For a new coming image with the timestamp $t$, we clone the latest marker pose and augment it to the state vector $x$ to track the camera pose. According to \cref{eq:clock}, the corresponding marker timestamp is $t + {t_d}$. The new cloned marker pose is:
\begin{equation}
    {x_{{c_{new}}}} = \left[ {\begin{array}{*{20}{c}}
    {{}_G^Mq\left( {t + {t_d}} \right)}\\
    {{}^G{p_M}\left( {t + {t_d}} \right)}
    \end{array}} \right]
\end{equation}


The state augmentation Jacobian with respect to ${\left[ {\begin{array}{*{20}{c}}
     {{}_G^M{q^T}}&{{}^Gp_M^T}&{{t_d}}
     \end{array}} \right]^T}$ is calculated as:
\begin{equation}
    {H_{aug}} = \left[ {\begin{array}{*{20}{c}}
    {{I_3}}&{{0_3}}&\omega \\
    {{0_3}}&{{I_3}}&{{}^G{v_M}}
    \end{array}} \right]
    \label{eq:H_aug}
\end{equation}

After the state augmentation is completed, we check the sliding window size and marginalize the oldest clone state if the window size exceeds $N$. The carefully selected feature points are used to update the poses over the sliding window and the position of the feature points. The feature measurement model can be written as:
\begin{equation}
    \begin{array}{l}
    {z_f} = \pi \left( {{}^C{p_f}} , \varsigma \right)\\
    {}^C{p_f} = {}_M^CR{}_G^MR\left( {{}^G{p_f} - {}^G{p_M}} \right) + {}^C{p_M}
    \end{array}
    \label{eq:feature_model}
\end{equation}

The subset of state variables related to ${z_f}$ is noted as\footnote{The camera intrinsic $\varsigma$ is omitted here because it does not affect the subsequent observability analysis in \cref{Observability Analysis}.}:
\begin{equation}
    {x_{{s}}} = {\left[ {\begin{array}{*{20}{c}}
    {{}_G^M{q^T}}&{{}^Gp_M^T}&{{}_M^C{q^T}}&{{}^Cp_M^T}&{{}^Gp_f^T}
    \end{array}} \right]^T}
\end{equation}

The feature measurement Jacobian is calculated as:
\begin{equation}
    \begin{array}{l}
    {H_f} = \frac{{\partial {z_f}}}{{\partial {}^C{p_f}}}{}_M^CR{}_G^MR\left[ {\begin{array}{*{20}{c}}
    {{J_1}}&{ - {I_3}}&{{J_2}}&{{}_M^GR{}_C^MR}&{{I_3}}
    \end{array}} \right]\\
    {J_1} = {\left[ {\left( {{}^G{p_f} - {}^G{p_M}} \right)} \right]_ \times }{}_M^GR\\
    {J_2} = {\left[ {\left( {{}^G{p_f} - {}^G{p_M}} \right)} \right]_ \times }{}_M^GR{}_C^MR
    \end{array}
    \label{eq:H_f}
\end{equation}

More details about feature detection, tracking, outlier rejection, triangulation, sliding window update scheme and covariance management can be found in \cite{geneva2020openvins}.

\subsection{Global Pose Measurement Update} \label{Mocap_Update}

The timestamp of the global pose measurements $t$, provided at the marker clock, are shifted by ${t_d}$, $t - {t_d}$. The corrected global pose measurement is used to update ${x_M}$. The global pose measurement model can be written as:
\begin{equation}
    {z_g} = \left[ {\begin{array}{*{20}{c}}
    {{}_G^Mq}\\
    {{}^G{p_M}}
    \end{array}} \right]
\end{equation}


The global pose measurement Jacobian with respect to ${\left[ {\begin{array}{*{20}{c}}
    {{}_G^M{q^T}}&{{}^Gp_M^T}
    \end{array}} \right]^T}$ is calculated as:
\begin{equation}
    {H_g} = \left[ {\begin{array}{*{20}{c}}
    {{I_3}}&{{0_3}}\\
    {{0_3}}&{{I_3}}
    \end{array}} \right]
    \label{eq:H_g}
\end{equation}

\section{Observability Analysis} \label{Observability Analysis}

System observability plays an important role in state estimation. To study the potential calibration failures, we perform observability analysis for the linearized system \cite{chen1990local} derived in the target-less calibration. To the best of our knowledge, this is the first time that a paper studies the observability of the spatial-temporal parameters between the camera and the marker. 

Since the state vector couples both motion variables and calibration parameters together by covariance matrix. It is expected that the success of calibration depends on motion profiles. Identifying the potential degenerate motion profiles that adversely affect the calibration accuracy can guide the calibration process in practice.

To concise the presentation, we do not consider clone states in the state vector. $\omega$ and ${}^G{v_M}$ are also neglected as their observablity property is consistent with the marker pose. And only one SLAM feature is kept. The results can be extended to general cases \cite{li2013high, hesch2013consistency}. The system state vector becomes:
\begin{equation}
    x = {\left[ {\begin{array}{*{20}{c}}
    {{}_G^M{q^T}}&{{}^Gp_M^T}&{{}_M^C{q^T}}&{{}^Cp_M^T}&{{t_d}}&{{}^Gp_f^T}
    \end{array}} \right]^T}
\end{equation}

The state transition matrix becomes:
\begin{equation}
    \Phi \left( {{t_k},{t_0}} \right) = \left[ {\begin{array}{*{20}{c}}
    A&{}\\
    {}&{{I_{13}}}
    \end{array}} \right]
\end{equation}

$A$ is defined in \cref{eq:transition}.

${H_{aug}}$ in \cref{eq:H_aug}, ${H_f}$ in \cref{eq:H_f} and ${H_g}$ in \cref{eq:H_g} are stacked to construct the general Jacobian of the state:
\begin{equation}
    {H_k} = \left[ {\begin{array}{*{20}{c}}
    {{I_3}}&{{0_3}}&{{0_3}}&{{0_3}}&\omega &{{0_3}}\\
    {{0_3}}&{{I_3}}&{{0_3}}&{{0_3}}&{{}^G{v_M}}&{{0_3}}\\
    {{J_1}}&{ - {I_3}}&{{J_2}}&{{}_M^GR{}_C^MR}&{{0_{3 \times 1}}}&{{I_3}}\\
    {{I_3}}&{{0_3}}&{{0_3}}&{{0_3}}&{{0_{3 \times 1}}}&{{0_3}}\\
    {{0_3}}&{{I_3}}&{{0_3}}&{{0_3}}&{{0_{3 \times 1}}}&{{0_3}}
    \end{array}} \right]
\end{equation}

The common factor $\frac{{\partial {z_f}}}{{\partial {}^C{p_f}}}{}_M^CR{}_G^MR$ in \cref{eq:H_f} is ignored here because it does not affect the observability analysis. Now the observability matrix would be constructed as \cite{chen1990local}:
\begin{equation}
    \begin{array}{c}
    O = {\left[ {\begin{array}{*{20}{c}}
     \cdots &{O_k^T}& \cdots 
    \end{array}} \right]^T}\\
    {O_k} = {H_k}\Phi \left( {{t_k},{t_0}} \right)\\
     = \left[ {\begin{array}{*{20}{c}}
    A&{{0_3}}&{{0_3}}&{{0_3}}&\omega &{{0_3}}\\
    {{0_3}}&{{I_3}}&{{0_3}}&{{0_3}}&{{}^G{v_M}}&{{0_3}}\\
    {{J_1}A}&{ - {I_3}}&{{J_2}}&{{}_M^GR{}_C^MR}&{{0_{3 \times 1}}}&{{I_3}}\\
    A&{{0_3}}&{{0_3}}&{{0_3}}&{{0_{3 \times 1}}}&{{0_3}}\\
    {{0_3}}&{{I_3}}&{{0_3}}&{{0_3}}&{{0_{3 \times 1}}}&{{0_3}}
    \end{array}} \right]
    \end{array}
\end{equation}

We note that for generic motions, ${O}$ is a time varying matrix, whose columns are linearly independent. At this point, we state that the spatial-temporal calibration parameters are observable with fully excited 6DoF motions.

However, under the special motion situation, the linear independent relationship is no longer maintained, resulting in some degrees of freedom of the calibration parameters becoming unobservable.

\begin{lemma}
If the frame $\{ M\} $ performs pure translation (no rotation) motion, ${}^C{p_M}$ is unobservable. The corresponding right null space of ${O}$ is:
\begin{equation}
    {N_1} = {\left[ {\begin{array}{*{20}{c}}
    {{0_{3 \times 9}}}&{{I_3}}&{{0_{3 \times 1}}}&{ - {{\left( {{}_M^GR{}_C^MR} \right)}^T}}
    \end{array}} \right]^T}
\end{equation}
\end{lemma}


\begin{proof}
The fact that ${N_1}$ is indeed the right null space of ${O}$ can be verified by multiplying ${O_k}$ with ${N_1}$. ${O_k}{N_1} = 0$ is hold for any $k$. And we note that ${N_1}$ is a constant matrix. Since there is no rotation, ${}_M^GR$ is a constant matrix. Hence, ${N_1}$ belongs to the right null space of ${O}$. ${N_1}$ indicates that the unobservable direction is ${}^C{p_M}$.
\end{proof}

\begin{lemma} \label{lemma2}
If the the frame $\{ M\}$ rotates around a constant axis $\omega_2$ during the generic translation motion, the unobservable directions depend on the projection of $\omega_2$ in the frame $\{ C\} $, and the corresponding right null space of ${O}$ is:
\begin{equation}
    {N_2} = {\left[ {\begin{array}{*{20}{c}}
    {{0_{1 \times 9}}}&{{{\left( {{}_M^CR\omega_2 } \right)}^T}}&{{0}}&{ - {{\left( {{}_M^GR\omega_2 } \right)}^T}}
    \end{array}} \right]^T}
\end{equation}
\end{lemma}


\begin{proof}
Similarly, we verify that ${O_k}{N_2} = 0$ is hold for any $k$. Since $\omega$ and $\omega_2$ are parallel at this setting, for any given $\omega_2$, the time derivative of ${}_M^GR\omega_2$ is given by:
\begin{equation}
    \frac{{d\left( {{}_M^GR\omega_2 } \right)}}{{dt}} = \left( {\frac{{d\left( {{}_M^GR} \right)}}{{dt}}} \right)\omega_2  = {}_M^GR{\left[ \omega  \right]_ \times }\omega_2  = 0
\end{equation}

This proves that ${N_2}$ is a constant matrix and belongs to the right null space of ${O}$. ${N_2}$ indicates that the unobservable directions are from ${}^C{p_M}$, and dependent on the non-zero components of ${}_M^CR\omega_2$, or ${}_M^CR\omega$.
\end{proof}

 There could be some other degeneration motion primitives that have not been considered, such as constant angular and linear velocities, constant angular velocity and linear accelerations. We can find these two are special cases for Lemma~\ref{lemma2}. In this paper, we do not derive all degeneration cases where the full column rank condition of $O$ breaks.

As a final remark, we note that the translation calibration parameter ${}^C{p_M}$ is more sensitive to different motions, compared to the rotation and temporal calibration parameter. These theory findings are important for the calibration, as these degenerate motions are likely to occur in practice, such as the planer motion of wheeled robot and the pure translation of flying robot. We run real-world experiments on random generic trajectories with full excitation to avoid these potential specific degenerate trajectories.

\section{Experiments}

We state again that the inputs of two proposed calibration methods are global pose measurements and monocular image stream.
Firstly, the observability analysis in \cref{Observability Analysis} is verified by generating these measurements in the simulation environment. Then the real-world datasets are used to test the calibration accuracy and consistency. The target-based method requires the calibration target to be located in the field of view of the image and geometric prior about the calibration target. Finally, an example of calibrating time-varying spatial-temporal parameters is presented with the online target-less method.

\subsection{Validation of the Observability Analysis} \label{Validation of the Observability Analysis}

The simulated environment includes randomly generated 3D points to be captured by images. The characteristics of the simulated sensors are consistent with those of the actual sensors used in the real-world. Global pose measurements are reported in 120Hz.  Images are received in 20Hz. The Gaussian noises of the sensors are generated and added into the synthetic measurements. \cref{fig:sim-img} shows the synthetic feature points and the corresponding reprojected points in one simulated image during the visual update process. The translation motion of the marker frame is simulated as a sinusoidal trajectory, which is widely used in calibration tasks \cite{li2014online, yang2020online, lv2022observability}.

To validate the observability assertion in \cref{Observability Analysis}, we set ${}_M^CR$ as ${I_3}$, and design five rotation motion cases.

\begin{itemize}

\item Case1: $\omega  = {\left[ {\begin{array}{*{20}{c}}
{0.4\cos \left( {1.5t} \right)}&{0.4\sin \left( t \right)}&0
\end{array}} \right]^T}$.

\item Case2: $\omega  = {\left[ {\begin{array}{*{20}{c}}
0&0&0
\end{array}} \right]^T}$.

\item Case3: $\omega  = {\left[ {\begin{array}{*{20}{c}}
{0.4}&0&0
\end{array}} \right]^T}$.

\item Case4: $\omega  = {\left[ {\begin{array}{*{20}{c}}
0&{0.5}&{0.6}
\end{array}} \right]^T}$.

\item Case5: $\omega  = {\left[ {\begin{array}{*{20}{c}}
{0.1}&{0.2}&{0.3}
\end{array}} \right]^T}$.

\end{itemize}

 The calibration results of these cases are presented in \cref{fig:sim}. The initial rotation error is ${\left[ {\begin{array}{*{20}{c}}
{{{20}^ \circ }}&{{{20}^ \circ }}&{ - {{20}^ \circ }}
\end{array}} \right]^T}$. The initial translation error is ${\left[ {\begin{array}{*{20}{c}}
{ - 5}&{15}&{ - 10}
\end{array}}\right]^T}$cm. The initial time offset error is 50 ms. Case1 corresponds to the generic motion with full excitation. It is clear that the estimation errors of all calibration parameters converge perfectly to near zero within 10s. All calibration parameters are observable in this case. Case 2 corresponds to a pure translation (no rotation) motion. The estimation error of the translation calibration parameter and its $1\sigma $ bound can not approach 0, thus this parameter is unobservable. While the rotation and temporal calibration parameters are still observable. Case3, Case4, and Case5 correspond to the constant axis rotational motion. The non-zero components of this axis indicate the unobservable directions. For example, the rotation axis of Case3 only has non-zero component in the $x$-axis. Thus, the $x$-direction of the translation calibration parameter is unobservable, yet $y$ and $z$ direction are still observable, as shown in \cref{fig:sim}. The similar analysis also applies to Case 4 and Case 5.

\subsection{Real-World Experiments}

\begin{figure*}
  \centering
  \begin{subfigure}[t]{0.24\linewidth}
    \includegraphics[width=\textwidth, height=1.0\textwidth]{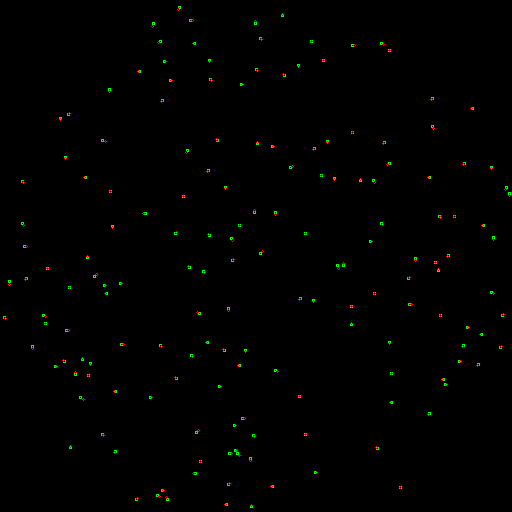}
    \caption{Tracked features and reprojected features in the simulation.}
    \label{fig:sim-img}
  \end{subfigure}
  \hfill
  \begin{subfigure}[t]{0.24\linewidth}
    \includegraphics[width=\textwidth, height=1.0\textwidth]{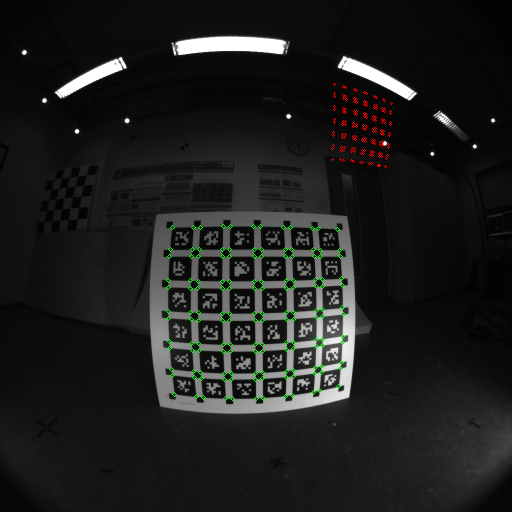}
    \caption{The 1st iteration.}
    \label{fig:1 iter}
  \end{subfigure}
  \hfill
  \begin{subfigure}[t]{0.24\linewidth}
    \includegraphics[width=\textwidth, height=1.0\textwidth]{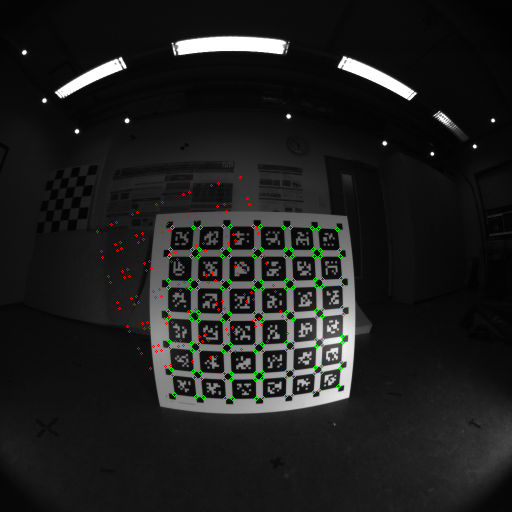}
    \caption{The 2nd iteration.}
  \end{subfigure}
  \hfill
  \begin{subfigure}[t]{0.24\linewidth}
    \includegraphics[width=\textwidth, height=1.0\textwidth]{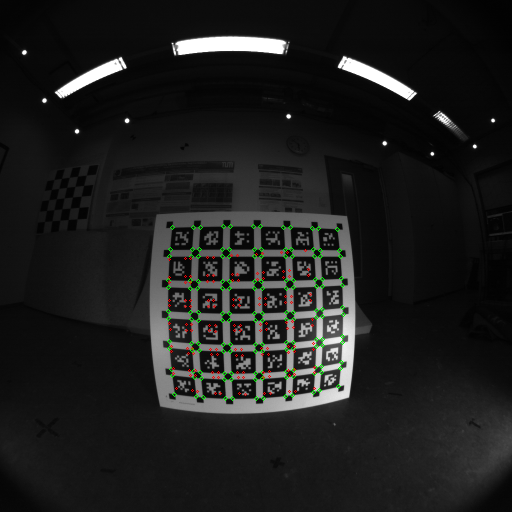}
    \caption{The 3rd iteration.}
  \end{subfigure}

  \begin{subfigure}[t]{0.24\linewidth}
    \includegraphics[width=\textwidth, height=1.0\textwidth]{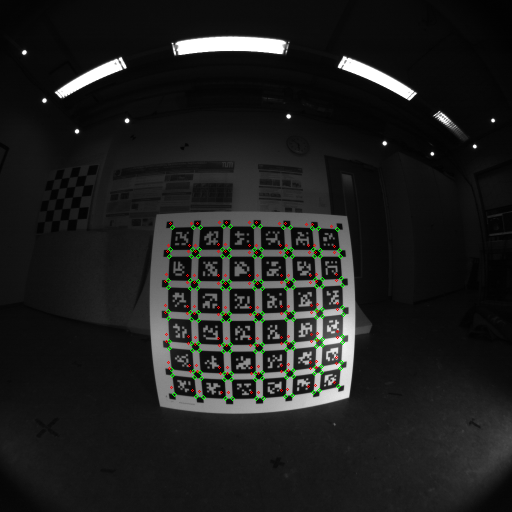}
    \caption{The 4th iteration.}
  \end{subfigure}
  \hfill
  \begin{subfigure}[t]{0.24\linewidth}
    \includegraphics[width=\textwidth, height=1.0\textwidth]{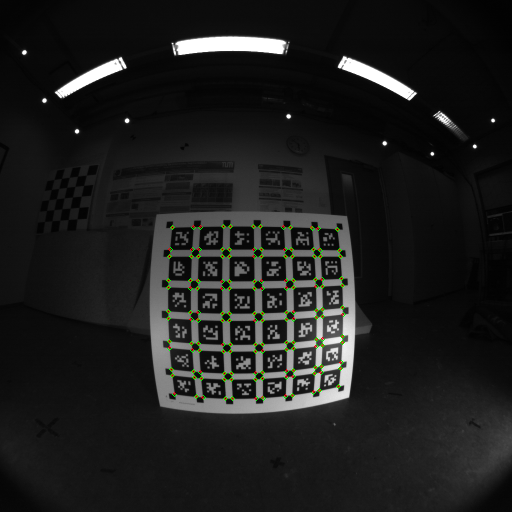}
    \caption{The 5th iteration.}
    \label{fig:5 iter}
  \end{subfigure}
  \hfill
  \begin{subfigure}[t]{0.24\linewidth}
    \includegraphics[width=\textwidth, height=1.0\textwidth]{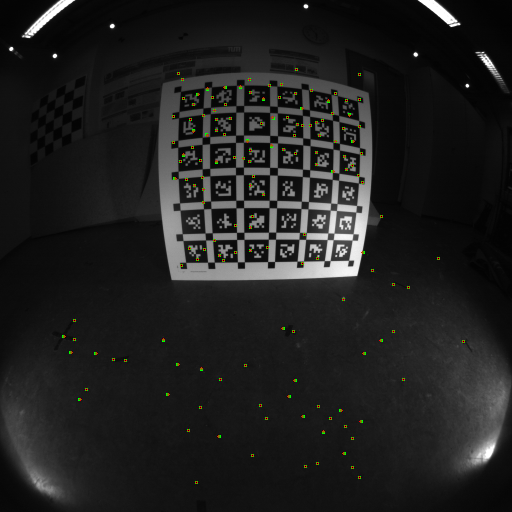}
    \caption{Image update in the environment with calibration target.}
    \label{fig:img-with-target}
  \end{subfigure}
  \hfill
  \begin{subfigure}[t]{0.24\linewidth}
    \includegraphics[width=\textwidth, height=1.0\textwidth]{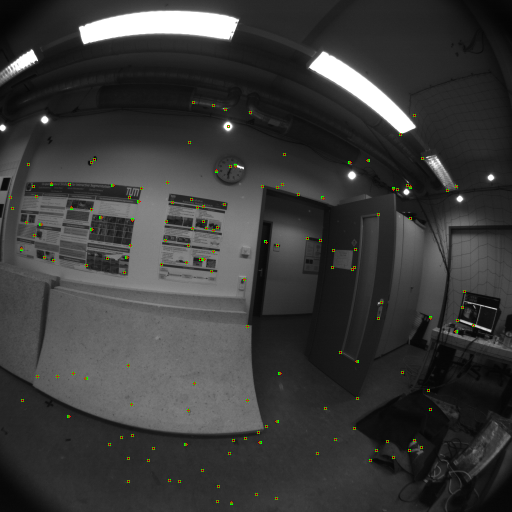}
    \caption{Image update in the environment without calibration target.}
    \label{fig:img-without-target}
  \end{subfigure}
  \caption{Expected feature positions (green) and predicted feature positions (red) in the image.}
  \label{fig:images}
\end{figure*}

\begin{figure*}
  \centering
  \begin{subfigure}{0.2\linewidth}
    \includegraphics[width=\textwidth, height=2.4\textwidth]{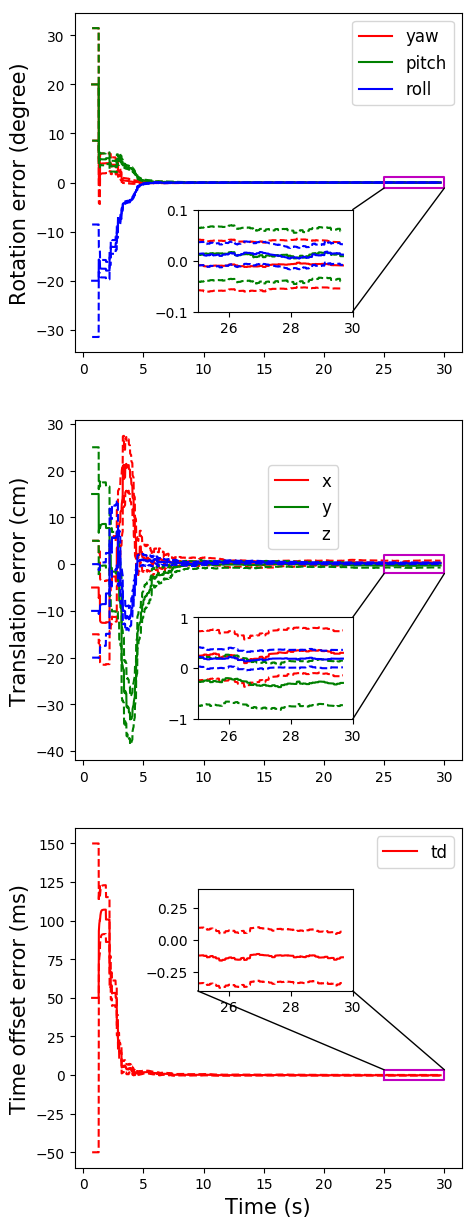}
  \end{subfigure}
  \hfill
  \begin{subfigure}{0.19\linewidth}
    \includegraphics[width=\textwidth, height=2.526\textwidth]{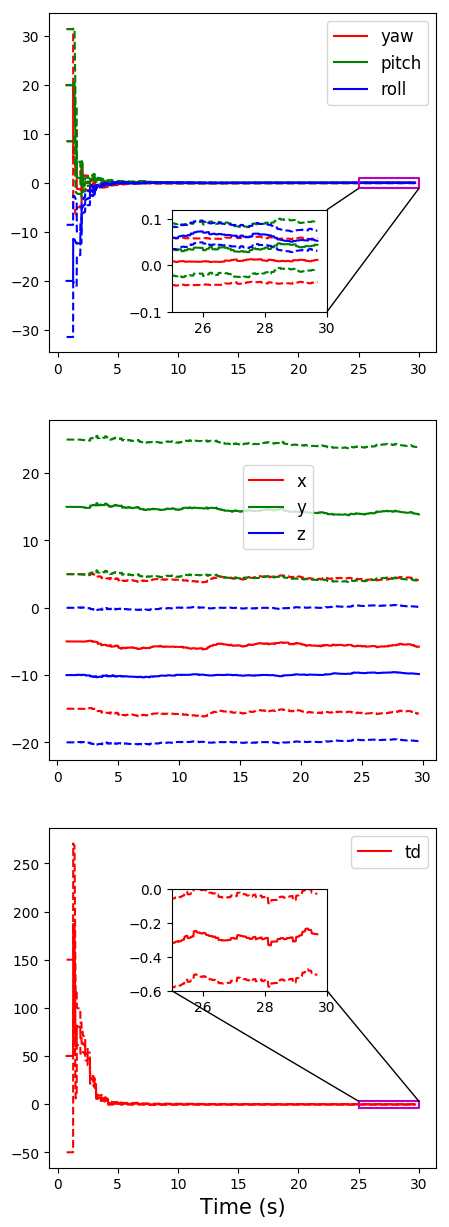}
  \end{subfigure}
  \hfill
  \begin{subfigure}{0.2\linewidth}
    \includegraphics[width=\textwidth, height=2.4\textwidth]{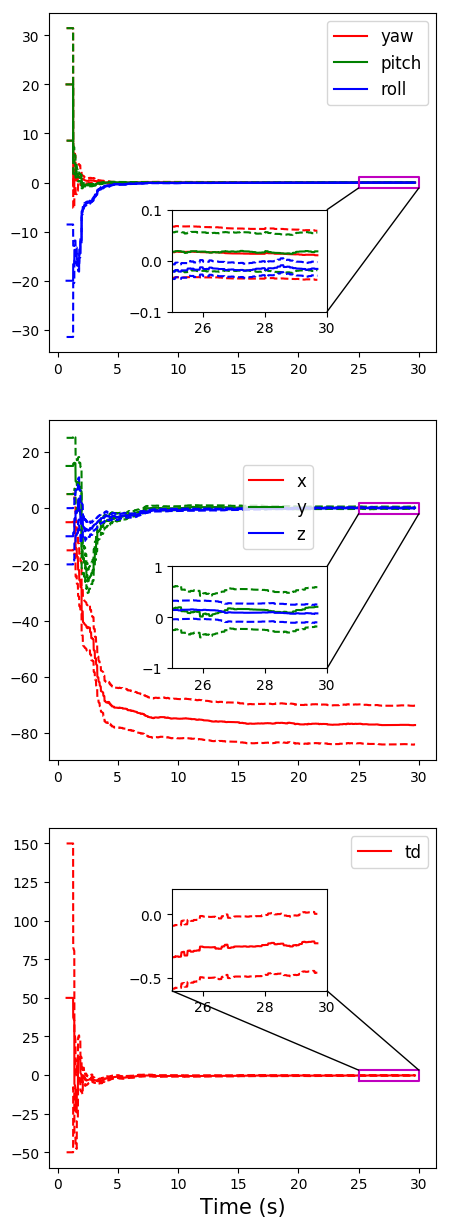}
  \end{subfigure}
  \hfill
  \begin{subfigure}{0.2\linewidth}
    \includegraphics[width=\textwidth, height=2.4\textwidth]{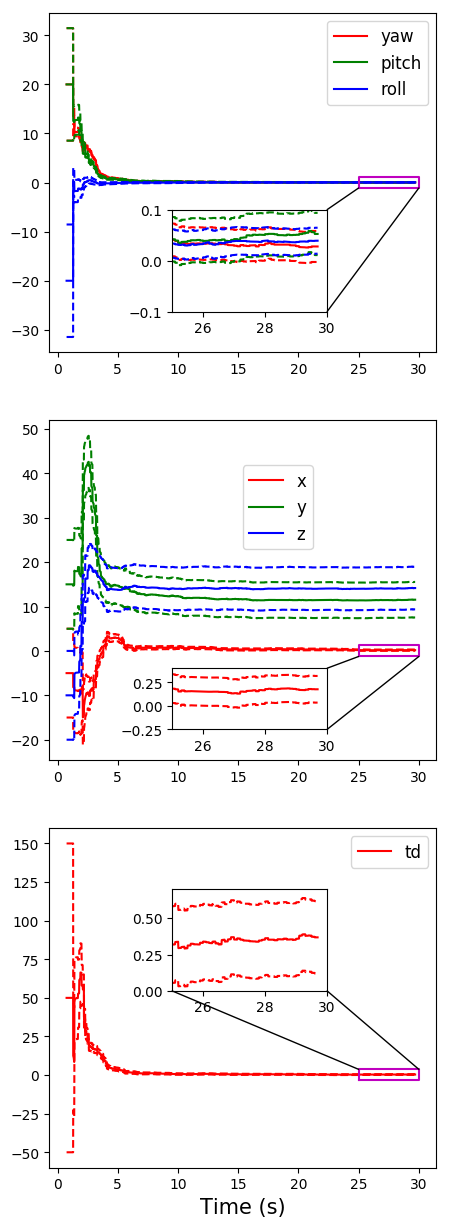}
  \end{subfigure}
  \hfill
  \begin{subfigure}{0.19\linewidth}
    \includegraphics[width=\textwidth, height=2.526\textwidth]{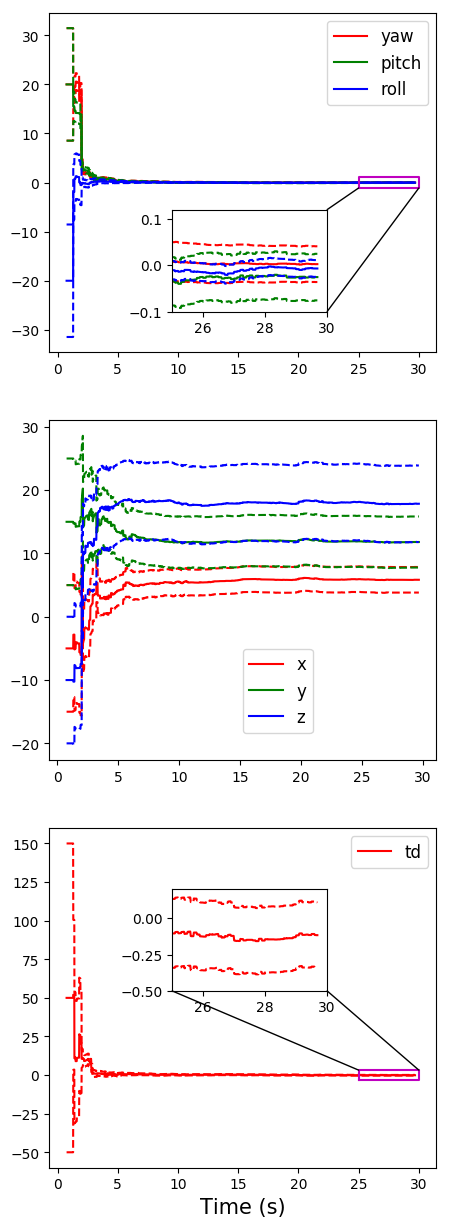}
  \end{subfigure}
  
  \caption{Errors (solid lines) and $1\sigma $ bounds (dashed lines) of the spatial-temporal calibration parameters. $x$-axis represents time in seconds. Left to right corresponds to Case1 to Case5 in \cref{Validation of the Observability Analysis}. The estimation error of the rotation and temporal calibration parameters perfectly approach to zero for any cases. While the convergence results of the translation calibration parameter are varied from case to case.}
  \label{fig:sim}
\end{figure*}

\begin{figure}
  \centering
  \includegraphics[width=0.46\textwidth]{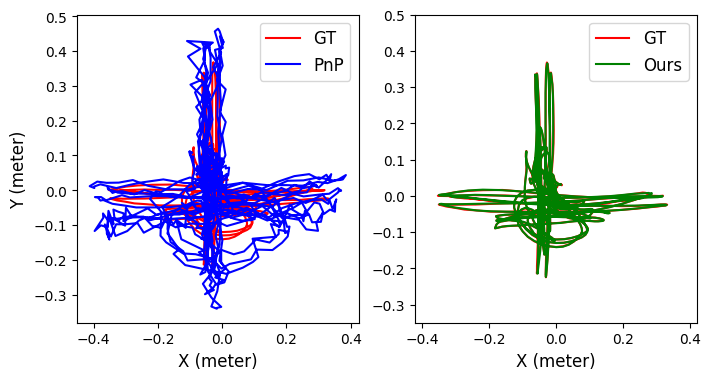}
  \caption{ $imu1$ is used. GT: groundtruth trajectory output from motion capture system. PnP: camera trajectory output from PnP algorithm. Ours: refined camera trajectory ${}_{{C_i}}^WT,{\rm{ }}i = 1 \cdots N$. }
  \label{traj}
\end{figure}

\begin{figure}
  \centering
  \includegraphics[width=0.46\textwidth]{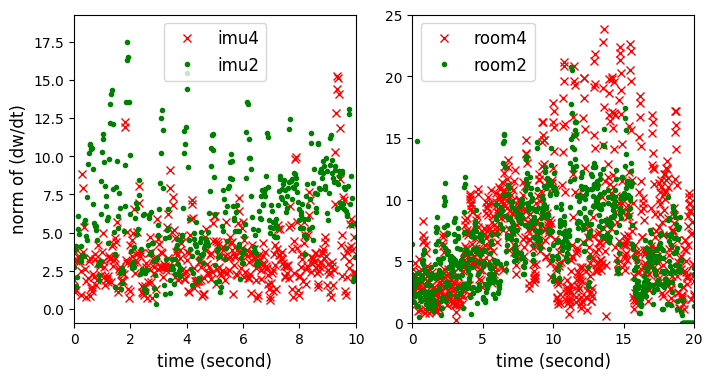}
  \caption{Norm of ${{d\omega } \mathord{\left/
 {\vphantom {{d\omega } {dt}}} \right.
 \kern-\nulldelimiterspace} {dt}}$.}
  \label{norm}
\end{figure}

\begin{table*}
  \caption{Average RMSE of the calibration results (mean value $\pm$ standard deviation) over 50 Monte-Carlo trials. Method1: target-less method. Method2: target-based method. L: left camera is used. R: right camera is used.}
  \centering
  \begin{tabular}{@{}ccccccc@{}}
    \toprule
    \multirow{2}{*}{Sequence} & \multicolumn{2}{c}{Rotation (deg)} & \multicolumn{2}{c}{Translation (cm)} & \multicolumn{2}{c}{Time offset (ms)} \\
    \cmidrule(lr){2-3} \cmidrule(lr){4-5} \cmidrule(lr){6-7}
    & {Method1} & {Method2} & {Method1} & {Method2} & {Method1} & {Method2} \\
    \midrule
    imu1 (L) & 0.124 $\pm$ 0.051 & 0.032 $\pm$ 4.74e-05 & 0.572 $\pm$ 0.126 & 0.103 $\pm$ 1.65e-05  & 0.543 $\pm$ 0.128 & 0.339 $\pm$ 0.00e-05 \\
    imu2 (L) & 0.142 $\pm$ 0.043 & 0.035 $\pm$ 4.63e-07 & 0.336 $\pm$ 0.076 & 0.090 $\pm$ 0.00e-07  & 0.149 $\pm$ 0.059 & 0.300 $\pm$ 0.00e-07 \\
    imu3 (L) & 0.074 $\pm$ 0.038 & 0.048 $\pm$ 0.00e-07 & 0.686 $\pm$ 0.141 & 0.146 $\pm$ 0.00e-07  & 0.088 $\pm$ 0.069 & 0.757 $\pm$ 0.00e-07 \\
    imu4 (L) & 0.083 $\pm$ 0.053 & 0.065 $\pm$ 3.91e-07 & 1.014 $\pm$ 0.115 & 0.125 $\pm$ 0.00e-07  & 1.156 $\pm$ 0.144 & 0.960 $\pm$ 0.00e-07 \\
    \hline\hline
    imu1 (R) & 0.075 $\pm$ 0.024 & 0.027 $\pm$ 9.97e-07 & 1.040 $\pm$ 0.228 & 0.085 $\pm$ 0.00e-07  & 0.432 $\pm$ 0.132 & 0.335 $\pm$ 0.00e-07 \\
    imu2 (R) & 0.180 $\pm$ 0.044 & 0.034 $\pm$ 0.00e-07 & 0.465 $\pm$ 0.270 & 0.075 $\pm$ 0.00e-07  & 0.161 $\pm$ 0.082 & 0.305 $\pm$ 0.00e-07 \\
    imu3 (R) & 0.125 $\pm$ 0.051 & 0.038 $\pm$ 0.00e-07 & 0.719 $\pm$ 0.101 & 0.136 $\pm$ 0.00e-07  & 0.091 $\pm$ 0.096 & 0.766 $\pm$ 0.00e-07 \\
    imu4 (R) & 0.087 $\pm$ 0.039 & 0.050 $\pm$ 3.22e-07 & 1.077 $\pm$ 0.119 & 0.132 $\pm$ 0.00e-07  & 1.449 $\pm$ 0.147 & 0.955 $\pm$ 0.00e-07 \\
    \bottomrule
  \end{tabular}
  \label{tab:test1}
\end{table*}

Firstly we present the rationale of dataset selection for real-world experiments.
For the target-less method, the simulation experiments in \cref{Validation of the Observability Analysis} show that it is advised to choose the fully excited 6DoF trajectory. The experiments in \cite{yang2020online} also inspire us to utilize the fully excited hand-held TUM-VI Dataset \cite{schubert2018tum} instead of under-actuated dataset, such as EuRoC MAV Dataset \cite{burri2016euroc}. TUM-VI Dataset contains multiple sequences with or without calibration target. Each sequence provides images at 20Hz, global pose measurements at 120Hz. These raw measurements together with IMU measurements are post-processed to ensure time-synchronization. Thus it is convenient to set the time offset by manually shifting the timestamps of the global pose measurements with a certain value. The shifted time offset is the reference value of the temporal parameter. As \cite{schubert2018tum} has leveraged IMU to align the marker frame to the IMU frame, the transformation from IMU to camera \cite{sommer2020efficient}, is also the reference value of the interested spatial parameter.

For each selected dataset, we run the specific calibration method multiple times to examine the statistical properties. Reference value is perturbed to perform a Monte-Carlo trial. The perturbed calibration parameters are set as initial calibration guess. Random errors drawn from zero-mean Gaussian distributions are added to reference values. For rotation and translation parameter, $1\sigma $ values of the error distribution along each axis are ${20^ \circ }$ and 10 cm respectively. For temporal parameter, the $1\sigma $ value is set as 50 ms.

\subsubsection{Environments with target}

Sequence $\left\{ {imu1 \sim imu4} \right\}$ is selected because the environments of these datasets contain the calibration target.

\cite{furrer2018evaluation}
can not work for these sequences due to the relatively large trajectory noise output by PnP algorithm, as shown in \cref{traj}. The absolute trajectory error (ATE) of the PnP trajectory is 7.29 cm, while the optimized trajectory of our target-based method has an ATE of only 0.28 cm. Clearly, the accuracy of camera trajectory has significantly improvement by fully utilizing the raw measurements.
Additional comparison results are provided in \cref{sec:additional_results} of supplementary material.

To visualize the estimation accuracy of the calibration parameters of the target-based method, the predicted feature position linked with calibration parameters is defined as:
\begin{equation}
    \begin{array}{l}
    z = \pi \left( {{}^C{p_f}} , \varsigma \right)\\
    {}^C{p_f} = {}_M^CT{}_G^MT\left( {t + {t_d}} \right){}_W^GT{}^W{p_f}\\
    \end{array}
    \label{eq:reprojection}
\end{equation}

Where ${f}$ denotes the AprilTag corner. $t$ is the image timestamp. ${{}_W^GT}$, ${{}_M^CT}$, ${{t_d}}$, and ${\varsigma}$ are variables from \cref{eq:opt_var}.

For a specific run of the target-based method, the iterative update results are visualized from \cref{fig:1 iter} to \cref{fig:5 iter}. After 5 iterations, all predicted feature positions are perfectly close to expected feature positions. \cref{fig:img-with-target} shows the feature points update of the target-less method. The predicted feature position is obtained via \cref{eq:feature_model}.

When using the left camera, the RMSE of the calibration results are shown in \cref{tab:test1}. As expected, the calibration accuracy and consistency of the target-based method are better than the target-less method. When using the right camera, the corresponding results are also shown in \cref{tab:test1}. Both calibration methods demonstrate similar accuracy and consistency for left and right camera. 

Compared with the target-based method, the target-less method's accuracy is affected by imperfect visual feature tracking and numerical precision of the triangulation process of visual landmarks. In addition, the target-less method is an online estimator, which can not use all available measurements simultaneously. 

It is worth noting that the dataset itself or the trajectory characteristic has impacts on the calibration accuracy for both methods. For example, the estimation accuracy of the translation calibration parameter of $imu2$ is better than that of $imu4$. Inspired by the observability analysis in \cref{Observability Analysis} and \cref{Validation of the Observability Analysis}, it is reasonable to examine the rotation excitation to reveal the behind reason. \cref{norm} depicts the norm of the angular velocity difference. $imu2$ has more sufficient rotation excitation, improving the observability of the translation calibration parameter.




\subsubsection{Environments without target}

To eliminate the impact of the calibration target on the accuracy of the target-less method, we conduct experiments on the sequence $\left\{ {room1 \sim room6} \right\}$ without calibration target. The target-based method can not work at this setting. 

The calibration results of the target-less method are shown in \cref{tab:test3}. Compared with the sequence with calibration target (see \cref{tab:test1}), the estimation of the calibration parameter does not incur loss of performance without the calibration target in the field of view. The calibration accuracy is still impacted by the trajectory itself. For example, the estimation accuracy of the translation calibration parameter of $room4$ is better than that of $room2$. \cref{norm} shows that $room4$ has more sufficient rotation excitation.

\begin{table}
  \caption{Average RMSE (L / R) of the calibration results over 50 Monte-Carlo trials. L: left camera. R: right camera. The units for rotation, translation and time offset are in deg, cm and ms.}
  \centering
  \begin{tabular}{@{}cccc@{}}
    \toprule
    Sequence & Rotation & Translation & Time offset \\
    \midrule
    room1 & 0.033 / 0.056 & 0.681 / 0.584 & 0.101 / 0.073 \\
    room2 & 0.136 / 0.136 & 0.860 / 0.758 & 0.957 / 0.930 \\
    room3 & 0.036 / 0.057 & 0.657 / 0.550 & 1.298 / 1.264 \\
    room4 & 0.042 / 0.043 & 0.315 / 0.385 & 0.633 / 0.588 \\
    room5 & 0.033 / 0.067 & 0.566 / 0.484 & 0.398 / 0.411 \\
    room6 & 0.161 / 0.180 & 0.765 / 0.708 & 0.601 / 0.696 \\
    \bottomrule
  \end{tabular}
  \label{tab:test3}
\end{table}


For all the results presented so far, the spatial-temporal parameters are assumed to be constant, which is also the most common scenario in practice. Considering the vibration or morphology change of the robot platform \cite{falanga2018foldable} and clock drift during the running, it is also worth investigating the calibration of time-varying spatial-temporal parameters, a more challenge scenario. $room4$ is used here for test. To construct time-varying spatial parameters, the global pose measurements are perturbed. ${}_M^{M'}T$ is the designed perturbation. The spatial parameters are changed accordingly.
\begin{equation}
    {}_G^{M'}T = {}_M^{M'}T{}_G^MT {\quad\quad} {}_C^{M'}T = {}_M^{M'}T{}_C^MT
\end{equation}

The time-vary temporal parameter is constructed more straightforward by changing the timestamps of the global pose measurements with designed time-vary values.

The target-based method can not work as it includes constant calibration parameters in state vector. And the requirement of facing the calibration target makes it impractical during the large change of calibration parameters. While EKF-based target-less method could handle dynamic change of state naturally, even without the prior knowledge about such change. As shown in \cref{fig:vary}, the time-varying quantity of spatial-temporal parameter is designed to change linearly with time. 
The initial rotation and translation errors along each axis are ${20^ \circ }$ and 10 cm respectively. The initial time offset error is 60 ms.
Despite the significant estimation errors at the beginning, the target-less method could quickly converge to the groundtruth value and accurately track the time-varying change. After 10s, the average tracking RMSE of the rotation change, the translation change and the time offset change are ${1.754^ \circ }$, 1.346 cm and 4.151 ms respectively. Once dynamic change stage is over, these small errors mean that good initial guess is provided for follow-up constant parameters calibration.

\begin{figure}
  \centering
  \includegraphics[width=0.478\textwidth, height=0.18\textwidth]{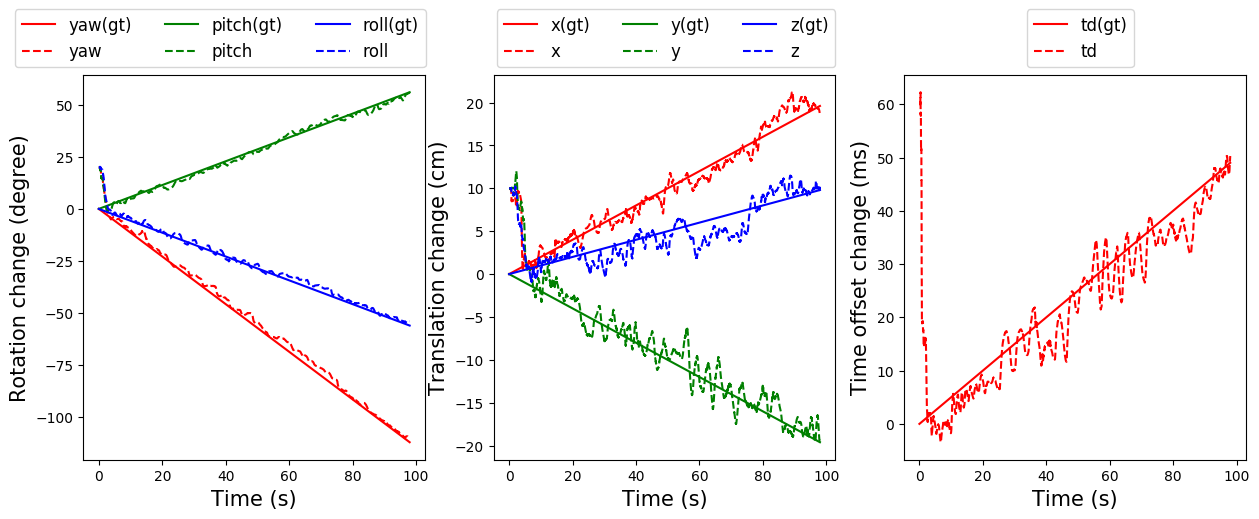}
  \caption{Groundtruth (solid lines) and estimation (dashed lines) of the time-varying change of the spatial-temporal parameters.}
  \label{fig:vary}
\end{figure}

%% file: sec/3_final.tex
\section{CONCLUSIONS}

In this work, we propose two novel calibration methods to estimate the spatial-temporal parameters between the camera and the global pose sensor. 
One is a target-based method, it adopts offline full-batch nonlinear least squares optimization.
Another is a target-less method based on an online EKF estimator. 
The observability analysis of the target-less method shows that the calibration parameters are observable when the system is fully excited by 6DoF movements. 
Real-world experiments demonstrate both methods provide accurate and reliable calibration results when traditional hand-eye calibration fails to work.
Moreover, the ability of capturing time-varying parameters, rarely studied in literature, is verified successfully for the target-less method. 
Proposed methods can be easily extended to other global pose sensors besides motion capture system, and different camera models. In the future, we plan to improve the accuracy of the target-less method using sliding window optimization.

%% file: sec/X_suppl.tex
\clearpage
\setcounter{page}{1}
\maketitlesupplementary

\section{Analytical on-manifold Jacobians for the target-based method}
\label{sec:Jacobians}

The optimization function (\cref{eq:min}) contains two types of measurement residual, namely pixel measurement residual and global pose measurement residual. The Jacobians of these residuals with respect to the optimization variables are provided here. On-manifold formulation of the optimization variables, like SE(3) transformations, allows us to easily calculate analytical Jacobian which is more accurate and computational efficient than numerical differentiation.

\subsection{Jacobians of pixel measurement residual}

Firstly, we analyze the Jacobians involved in the pixel measurement residual ${r_{ij}}$:
\begin{equation}
    \begin{array}{l}
    {r_{ij}} = \pi \left( {{}^{{C_i}}{p_{{f_j}}},\varsigma } \right) - {u_{ij}}\\
    {}^{{C_i}}{p_{{f_j}}} = {}_W^{{C_i}}T{}^W{p_{{f_j}}}
    \end{array}
\end{equation}

The subset of optimization variables related to ${r_{ij}}$ is noted as:
\begin{equation}
    {\chi_{{s_1}}}  = \left\{ {\begin{array}{*{20}{c}}
    {{}_{{C_i}}^WT}&{\varsigma}
    \end{array}} \right\}
\end{equation}

The Jacobians of the pixel residual ${r_{ij}}$ with respect to the 3D point in camera frame ${}^{{C_i}}{p_{{f_j}}}$ and the camera intrinsic $\varsigma$ are $\frac{{\partial {r_{ij}}}}{{\partial {}^{{C_i}}{p_{{f_j}}}}}$ and $\frac{{\partial {r_{ij}}}}{{\partial \varsigma }}$ respectively. Both are determined by the camera projection model \cite{usenko2018double, heng2013camodocal}. The Jacobian of the pixel residual ${r_{ij}}$ with respect to the camera pose ${}_{{C_i}}^WT$ is:
\begin{equation}
    \begin{array}{l}
    \frac{{\partial {r_{ij}}}}{{\partial {}_{{C_i}}^WT}} = \frac{{\partial {r_{ij}}}}{{\partial {}^{{C_i}}{p_{{f_j}}}}}\frac{{\partial {}^{{C_i}}{p_{{f_j}}}}}{{\partial {}_W^{{C_i}}T}}\frac{{\partial {}_W^{{C_i}}T}}{{\partial {}_{{C_i}}^WT}}\\
    \frac{{\partial {}^{{C_i}}{p_{{f_j}}}}}{{\partial {}_W^{{C_i}}T}} = {\left( {{}_W^{{C_i}}T{}^W{p_{{f_j}}}} \right)^ \odot }\\
    \frac{{\partial {}_W^{{C_i}}T}}{{\partial {}_{{C_i}}^WT}} =  - I
    \end{array}
    \label{eq:J_TWC}
\end{equation}

Where $\odot$ is an operator for the homogeneous coordinate \cite[Sec. 7.1.8]{barfoot2017state}.

In summary, the Jacobians of the pixel measurement residual ${r_{ij}}$ with respect to ${\chi_{{s_1}}}$ can be computed via \cref{eq:J_TWC} and $\frac{{\partial {r_{ij}}}}{{\partial \varsigma }}$.

\subsection{Jacobians of global pose measurement residual}

Next, we analyze the Jacobians involved in the global pose measurement residual ${r_{gi}}$ (\cref{eq:min}). To simplify the description, we define the following intermediate quantities:
\begin{equation}
    \begin{array}{l}
    {}_G^M\hat T \buildrel \Delta \over = {}_G^MT\left( {{t_i} + {t_d}} \right)\\
    {}_C^WT \buildrel \Delta \over = {}_{{C_i}}^WT\\
    {}_{{M_a}}^{{M_b}}\theta  \buildrel \Delta \over = Log\left( {{}_G^{{M_b}}T{}_G^{{M_a}}{T^{ - 1}}} \right)
    \end{array}
\end{equation}

Therefore
\begin{equation}
    \begin{array}{l}
    {r_{gi}} = Log\left( {{}_G^M\hat T{}_W^GT{}_C^WT{}_M^CT} \right)\\
    {}_G^M\hat T = Exp\left( {\lambda {}_{{M_a}}^{{M_b}}\theta } \right){}_G^{{M_a}}T\\
    \lambda  = {{\left( {{t_i} + {t_d} - {t_a}} \right)} \mathord{\left/
     {\vphantom {{\left( {{t_i} + {t_d} - {t_a}} \right)} {\left( {{t_b} - {t_a}} \right)}}} \right.
     \kern-\nulldelimiterspace} {\left( {{t_b} - {t_a}} \right)}}
    \end{array}
\end{equation}

The subset of optimization variables related to ${r_{gi}}$ is noted as:
\begin{equation}
    {\chi_{{s_2}}}  = \left\{ {\begin{array}{*{20}{c}}
    {{}_{{C_i}}^WT}&{{}_W^GT}&{{}_M^CT}&{{t_d}}
    \end{array}} \right\}
\end{equation}

The Jacobian of ${r_{gi}}$ with respect to ${}_M^CT$ is:
\begin{equation}
    \frac{{\partial {r_{gi}}}}{{\partial {}_M^CT}} = J_r^{ - 1}\left( {{r_{gi}}} \right)
    \label{eq:J_TMC}
\end{equation}

Where ${J_r}\left( \bullet \right)$ is the right Jacobian of SE(3) \cite{barfoot2017state}.

The Jacobian of ${r_{gi}}$ with respect to ${}_C^WT$ is:
\begin{equation}
    \frac{{\partial {r_{gi}}}}{{\partial {}_C^WT}} = J_r^{ - 1}\left( {{r_{gi}}} \right)Ad\left( {{}_M^C{T^{ - 1}}} \right)
    \label{eq:J_TCW}
\end{equation}

Where $Ad\left(  \bullet  \right)$ is the adjoint of SE(3) \cite{barfoot2017state}.

The Jacobian of ${r_{gi}}$ with respect to ${}_W^GT$ is:
\begin{equation}
    \frac{{\partial {r_{gi}}}}{{\partial {}_W^GT}} = J_r^{ - 1}\left( {{r_{gi}}} \right)Ad\left( {{{\left( {{}_C^WT{}_M^CT} \right)}^{ - 1}}} \right)
    \label{eq:J_TWG}
\end{equation}

The Jacobian of ${r_{gi}}$ with respect to ${}_G^M\hat T$ is:
\begin{equation}
    \frac{{\partial {r_{gi}}}}{{\partial {}_G^M\hat T}} = J_r^{ - 1}\left( {{r_{gi}}} \right)Ad\left( {{{\left( {{}_G^M\hat T{}_W^GT{}_C^WT{}_M^CT} \right)}^{ - 1}}} \right)
\end{equation}

The Jacobian of ${}_G^M\hat T$ with respect to $\lambda$ is:
\begin{equation}
    \frac{{\partial {}_G^M\hat T}}{{\partial \lambda }} = Ad\left( {Exp\left( {\lambda {}_{{M_a}}^{{M_b}}\theta } \right)} \right){J_r}\left( {\lambda {}_{{M_a}}^{{M_b}}\theta } \right){}_{{M_a}}^{{M_b}}\theta 
\end{equation}

The Jacobian of $\lambda$ with respect to ${t_d}$ is:
\begin{equation}
    \frac{{\partial \lambda }}{{\partial {t_d}}} = \frac{1}{{{t_b} - {t_a}}}
\end{equation}

Finally, through the chain rule, the Jacobian of ${r_{gi}}$ with respect to ${t_d}$ is calculated as:
\begin{equation}
    \frac{{\partial {r_{gi}}}}{{\partial {t_d}}} = \frac{{\partial {r_{gi}}}}{{\partial {}_G^M\hat T}}\frac{{\partial {}_G^M\hat T}}{{\partial \lambda }}\frac{{\partial \lambda }}{{\partial {t_d}}}
    \label{eq:J_td}
\end{equation}

In summary, the Jacobians of the global pose measurement residual ${r_{gi}}$ with respect to ${\chi_{{s_2}}}$ can be computed via \cref{eq:J_TMC}, \cref{eq:J_TCW}, \cref{eq:J_TWG} and \cref{eq:J_td}.

\section{Additional comparison results}
\label{sec:additional_results}

\begin{figure*}
  \centering
  \includegraphics[width=0.9\textwidth, height=0.39\textwidth]{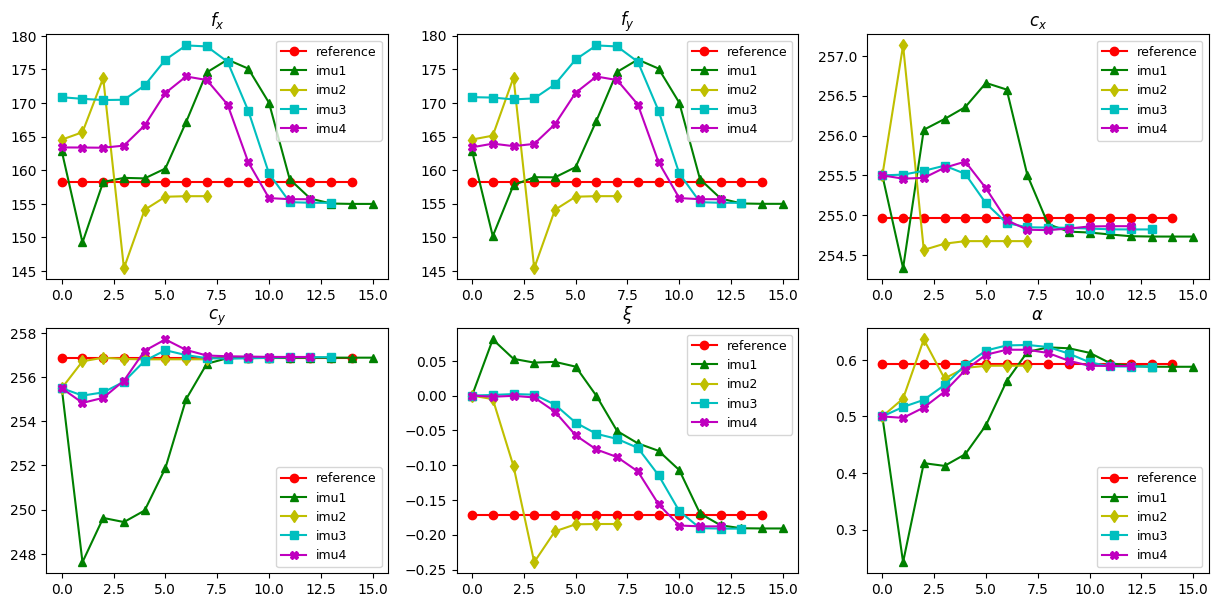}
  \caption{ Iterative process of calibrating left camera intrinsic from scratch. $x$-axis represents iteration steps. }
  \label{intrinsic1}
\end{figure*}

\begin{figure*}
  \centering
  \includegraphics[width=0.9\textwidth, height=0.39\textwidth]{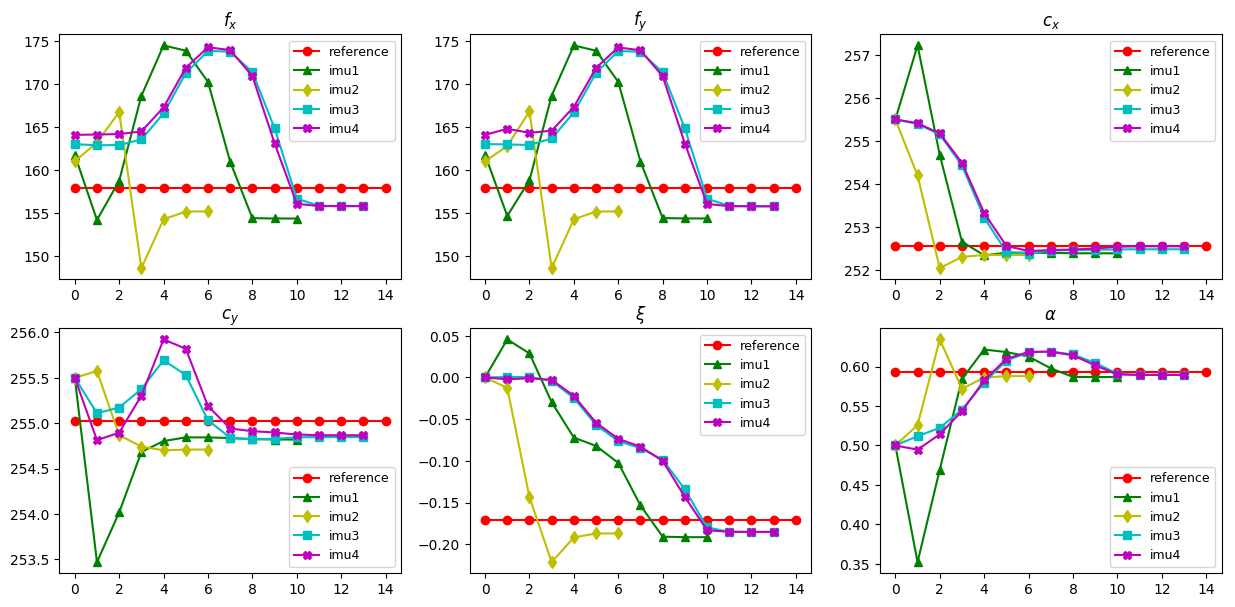}
  \caption{ Iterative process of calibrating right camera intrinsic from scratch. $x$-axis represents iteration steps. }
  \label{intrinsic2}
\end{figure*}

Compared to \cite{furrer2018evaluation}, our proposed target-based  method has another benefit, in addition to iterative optimization of camera trajectory. Prior to perform spatial-temporal hand-eye calibration, \cite{furrer2018evaluation} need to calibrate the camera intrinsic first. While our method does not require this step, as camera intrinsic is added to the optimization variables. This simultaneously calibration feature simplifies the calibration process. Moreover, \cite{furrer2018evaluation} may suffer from the fixed camera intrinsic. Environmental influences and camera motions may lead to unmodelled errors for camera intrinsic. To address this issue, our method finds the optimal camera intrinsic parameters that best fit all available measurements for each sequence.

\cref{intrinsic1} shows the iterative process of calibrating monocular camera intrinsic from scratch with our target-based method. Left camera is used for the selected sequence $\left\{ {imu1 \sim imu4} \right\}$ from TUM-VI Dataset \cite{schubert2018tum}, and double sphere camera model \cite{usenko2018double} is adopted. Regarding the initialization method and reference values for camera intrinsic parameters,
we refer to \cite{usenko2018double}. 
In \cref{intrinsic1}, all estimated intrinsic parameters converge near the reference values, with slightly difference for each sequence. When using the right camera, the corresponding results are shown in \cref{intrinsic2}. Final average reprojection error and position error in \cref{eq:min} are smaller than 0.1 pixel and 0.1 cm for left and right camera from each sequence. Results from \cref{intrinsic1} and \cref{intrinsic2} demonstrate the ability of calibrating optimal camera intrinsic from scratch for each sequence with the target-based method.